\documentclass{article}

\usepackage[preprint]{neurips_2026}

\usepackage[utf8]{inputenc} 
\usepackage[T1]{fontenc}    
\usepackage{hyperref}       
\usepackage{url}            
\usepackage{booktabs}       
\usepackage{amsfonts}       
\usepackage{nicefrac}       
\usepackage{microtype}      
\usepackage{xcolor}         
\usepackage{amsmath} 
\usepackage{graphicx} 
\usepackage{amssymb}
\usepackage{mathtools}
\usepackage{float}
\usepackage[dvipsnames]{xcolor}
\usepackage{pgfplots}
\usepgfplotslibrary{groupplots}
\pgfplotsset{compat=1.17}
\usepackage{rotating}
\usepackage{multirow}
\usepackage{makecell}

\usepackage{amsmath, amssymb, amsthm}
\newcommand{\tablescale}{0.9} 
\theoremstyle{plain}
\newtheorem{theorem}{Theorem}[section]
\newtheorem{lemma}[theorem]{Lemma}

\theoremstyle{definition}

\theoremstyle{remark}

\usepackage{placeins}

\usepackage[textsize=tiny]{todonotes}

\title{SFO: Learning PDE Solution Operators via Hilbert Spectral Basis}

\author{%
  Noam Koren \\
  Department of Computer Science\\
  Technion - Israel Institute of Technology\\
  \texttt{noam.koren@campus.technion.ac.il} \\
  \And
  Rafael Moschopoulos \\
  Department of Computer Science\\
  Princeton University \\
  \texttt{gm3460@princeton.edu} \\
  \AND
  Kira Radinsky \\
  Department of Computer Science\\
  Technion - Israel Institute of Technology\\
  \texttt{kiraradinsky@gmail.com} \\
  \And
  Elad Hazan \\
  Department of Computer Science\\
  Princeton University \\
  \texttt{elad.hazan@gmail.com}
}

\begin{document}

\maketitle
\begin{abstract}
Neural operators have emerged as a powerful framework for learning PDE solution operators directly from data. However, efficiently capturing the long-range, nonlocal interactions that characterize these operators remains challenging. We introduce Spectral Filtering Operator (SFO), a neural operator that parameterizes integral kernels in the Universal Spectral Basis (USB), a fixed global orthonormal basis derived from the eigenmodes of the Hilbert matrix in spectral filtering theory. We show that, for an important class of PDEs, the discrete Green's kernel function has a spatial Linear Dynamical System (LDS) structure, implying that it admits a compact USB approximation with an $\epsilon$-accuracy error bound. By learning only the spectral coefficients, SFO yields a highly efficient representation. Across six benchmarks spanning reaction-diffusion, fluid dynamics, and electromagnetics, SFO achieves state-of-the-art accuracy, reducing error by up to 40\% relative to strong baselines while using substantially fewer parameters.
\end{abstract}

\section{Introduction}
Many phenomena in physics, engineering, and natural science are governed by partial differential equations (PDEs). Classical numerical solvers (e.g., finite differences \cite{fdm}) are often expensive at high resolution, where very fine discretizations are needed \cite{tang2017global}. This has motivated neural operators: learned models that approximate the PDE solution operator, i.e., a mapping from an input function to the corresponding solution field. Once trained, a neural operator can amortize simulation cost across many queries, enabling fast inference for new PDE instances. Furthermore, unlike traditional approaches that require explicit knowledge of the governing equations, neural operators are data-driven and can learn the solution map directly from observations.



A central challenge in neural operators is efficiently capturing nonlocal interactions (e.g., FNO~\cite{li2020fourier}, MPNN~\cite{mpnn}). 
Many PDE solution operators exhibit long-range dependencies, where the solution at a point depends on distant regions of the domain (e.g., elliptic PDEs and fractional diffusion~\cite{greens_pde}). 
Capturing such behavior requires propagating information globally while maintaining a compact and efficient representation.


In this work, we introduce the \textbf{\textsc{Spectral Filtering Operator (SFO)}}, a neural operator that explicitly learns the underlying kernel, whereas most neural operators (e.g., FNO~\cite{li2020fourier}) parameterize the operator implicitly without explicitly learning the kernel. 
SFO parameterizes the kernel using the Universal Spectral Basis (USB), a fixed, globally supported orthogonal basis given by the eigenvectors of the Hilbert matrix. 
The kernel is represented as an expansion over the leading USB modes, and only the corresponding spectral coefficients are learned. 
The USB was originally proposed for modeling Linear Dynamical Systems (LDS)~\cite{hazan2017spectral}. 
Unlike Fourier bases, which bias toward smooth periodic functions, or learned bases, which may overfit, the USB provides a fixed global representation with rapid spectral decay.

This choice aligns with the structure of PDE solution operators. 
We analyze a representative setting in which the discrete PDE kernel exhibits an LDS structure and show that its USB expansion decays exponentially. 
As a result, an $\epsilon$-accurate approximation requires only logarithmically many leading modes, enabling \textsc{SFO} to capture long-range interactions efficiently with a compact representation.



Importantly, SFO is not limited to this LDS setting: the LDS analysis provides a theoretical explanation for an important class of PDE operators, while the architecture applies generally. Empirically, the same low-mode USB bias is effective across diverse PDE benchmarks beyond the LDS regime.

Across six PDE benchmarks, \textsc{SFO} achieves consistent state-of-the-art accuracy, with up to $40\%$ error reduction while using significantly fewer parameters than prior work. 

The contributions of this work are:
\begin{enumerate}

\item  We introduce a fixed,  spectral parameterization for neural-operator kernels based on the Universal Spectral Basis (USB), enabling compact representations of long-range interactions.

\item  We design \textsc{SFO}, a neural operator that realizes this parameterization.

\item  For an important class of PDEs, we prove LDS structure in the kernel and derive a USB truncation error bound showing that $\epsilon$-accuracy is achievable with $\tilde{\mathcal{O}}\big(\log(1/\epsilon)\big)$ USB modes.

\item We demonstrate consistent state-of-the-art performance over six diverse PDE benchmarks, achieving up to $40\%$ error reduction with significantly fewer parameters than prior work.
\end{enumerate}






An anonymized code repository is available.\footnote{https://anonymous.4open.science/r/STUNO-678F}

\section{Related Work}
\label{related_work}
Neural operator architectures primarily differ in how they represent the solution operator. We categorize approaches into five major classes: Coordinate-based, Fourier-based, Graph-based, Physics-Informed, and Explicit Kernel Representation. Our comparative analysis includes baselines for each class, covering the established architectures in the PDENNEval survey~\cite{pdenneval} (DeepONet, FNO, UNO, MPNN, PINO) as well as the recent kernel-learning method SVD-NO \cite{koren2025svd}, and the sequence-based methods MNO \cite{mno} and Transolver \cite{transolver}.


\textbf{Coordinate-Based.}
DeepONet~\cite{deeponet} and its variants (e.g., F-DeepONet~\cite{fdeeponet}) use a branch--trunk architecture to learn operators as a composition of input and coordinate embeddings, but scales poorly with resolution. 
MNO~\cite{mno} and Transolver~\cite{transolver} model the operator as a sequence, capturing global dependencies at the cost of higher computational complexity. 
We include these methods as representative baselines.

\textbf{Fourier-Based.}
FNO~\cite{li2020fourier} and its variants (e.g., WNO~\cite{tripura2022wavelet}, 
LNO~\cite{cao2303lno}, KNO~\cite{kno}, UNO~\cite{uno}) learn spectral multipliers in the Fourier domain using truncated modes, enabling efficient global interactions but biasing toward low-frequency structure. 
UNO~\cite{uno} extends this with U-Net-style architectures to improve multiscale representation. 
We include FNO and UNO as baselines.


\textbf{Graph-Based.}
For unstructured domains, GNN-based models (e.g., GNO~\cite{gkn}, 
GINO~\cite{gino}, MPNN~\cite{mpnn}) approximate the operator via message passing on graph edges. 
They handle irregular geometries but can struggle with long-range interactions and oversmoothing; we use MPNN as a baseline.

\textbf{Physics-Informed.}
PINO~\cite{pino} incorporates a PDE residual term into the loss, enforcing physical consistency during training.
We include PINO as a baseline.


\textbf{Explicit Kernel Representation.}
These methods parameterize the operator kernel explicitly. SVD-NO~\cite{koren2025svd} learns a low-rank kernel decomposition, which can be expressive but may be more sensitive to data coverage and overfitting; we include it as a comparison.

A parallel line of research introduced the \textbf{Universal Spectral Basis (USB)}, derived from the eigenvectors of the Hilbert matrix, motivated by learning long-range dependencies in LDS \cite{hazan2017spectral}. The USB provides a fixed orthonormal set of globally supported modes with rapid spectral decay, offering an attractive basis for compact spectral representations. Unlike Fourier bases, USB modes are not frequency-ordered and exhibit rapid decay for a broad class of LDS, while remaining fixed and data-independent unlike learned bases.

SFO is an addition to the \emph{explicit kernel representation} operator category: as in kernel-learning methods, it learns the kernel explicitly via a constrained expansion. Specifically, SFO expands the kernel in the USB and learns only the expansion coefficients. Conceptually, SFO adapts USB spectral filtering from sequence modeling to PDE operator learning, yielding a compact global kernel.


\section{Background}
\subsection{Neural Operator}
\label{sec:neural_ops}

\textbf{Operator Learning.}
Let $\mathcal D \subset \mathbb{R}^{d_x}$ be a bounded domain. 
Consider input and output function spaces
\[
\mathcal A := L^2(\mathcal D; \mathbb{R}^{d_a})\text{, } 
\mathcal U := L^2(\mathcal D; \mathbb{R}^{d_u}),
\text{ and a latent function space, }
\mathcal V := L^2(\mathcal D; \mathbb{R}^{d_v}).
\]
An operator $\mathcal G: \mathcal A \to \mathcal U$ maps an input field $a \in \mathcal A$ to an output field $u = \mathcal G(a) \in \mathcal U$ (e.g., the PDE solution). 
The goal is to learn $\mathcal G$ from data $\{(a_j, u_j)\}_{j=1}^N$ with $u_j = \mathcal G(a_j)$~\citep{Neuralop}.

\textbf{Neural Operator Architecture.}
A neural operator parameterizes
\(
\mathcal G_\theta:\mathcal A\to\mathcal U
\). 
Kernel-based operators are typically composed of a lifting map, operator layers, and a projection:
\[
\mathcal G_\theta
=
Q\circ
\bigl[\,\mathrm{layer}_{T-1}(\mathcal{K}_\theta^{T-1})\,\bigr]
\circ \cdots \circ
\bigl[\,\mathrm{layer}_{0}(\mathcal{K}_\theta^{0})\,\bigr]
\circ P,
\label{eq:no_arch}
\]

Here:
\begin{itemize}
\item $P: \mathbb{R}^{d_a} \to \mathbb{R}^{d_v}$ and $Q: \mathbb{R}^{d_v} \to \mathbb{R}^{d_u}$ are pointwise lifting and projection maps.
\item Each layer applies an integral operator $\mathcal K_\theta^{t}: \mathcal V \to \mathcal V$.
\item The operator $\mathcal K_\theta^{t}$ is defined by a kernel $\kappa_\theta^{t}$ as
\[
\bigl[\mathcal K_\theta^{t}(a)v\bigr](x)
=
\int_{\mathcal D}
\kappa_\theta^{t}\!\bigl(x, a(x), x', a(x')\bigr)\, v(x')\, dx'.
\]
\end{itemize}

\subsection{Universal Spectral Basis (USB)}
\label{subsec:sfo_spectral_filtering}
The \emph{Universal Spectral Basis} (USB) of \cite{hazan2017spectral} shows that a long-range structure in LDS can be captured using \textit{spectral filters} derived from the eigenvectors of the Hilbert matrix. 
Subsequent works extend USB spectral filtering to wider classes of nonlinear dynamics \cite{ dogariu2025universal}. 

\textbf{The Hilbert eigenvectors.}
The eigenvectors of the Hilbert matrix $H$, denoted by $\{\phi_l\}_{l=1}^n$, form a fixed, globally supported orthonormal system that exhibits:

1. \textit{Global:} Each eigenvector spans the entire domain.

2. \textit{Fast spectral decay:} The eigenvalues decay rapidly; beyond $l\!\approx\!20$ modes they are numerically negligible \cite{hazan2022introduction}, allowing few modes to capture global structure.

3. \textit{Universality:}
Hilbert filters form an $\ell_2$-complete basis.

\textbf{Discrete USB and its continuous interpretation.}
The USB modes \(\{\phi_\ell\}_{\ell=1}^n\) are the \(\ell_2\)-orthonormal eigenvectors of the Hilbert matrix \(H^{(n)}\).
Let \(\mathcal D=[0,1]\) be discretized by a uniform grid of \(n\) points $\{x_i\}_{i=1}^n$ with spacing $\Delta x=1/n$. Define the associated piecewise-constant functions
\(
\phi_\ell^{(n)}(x):=\phi_\ell[i]
\)
for
$x\in[x_i,x_{i+1}).$
Then \(\{\phi_\ell^{(n)}\}_{\ell=1}^n\) is orthogonal in \(L^2(\mathcal D)\), with
\(
\langle \phi_\ell^{(n)}, \phi_k^{(n)} \rangle = \Delta x\, \delta_{\ell k}.
\)

\section{The Spectral Filtering Operator (SFO)}
\label{sec:SFO}
\subsection{From LDS-Structured PDE Kernels to SFO Kernel}
\label{sec:sfo_architecture}

\textbf{Generic integral layer.} 
Let $\mathcal D\subset\mathbb R$ be a bounded spatial domain. A standard kernel-based
neural operator layer has the form
\begin{equation}
\bigl[\mathcal K_\theta v\bigr](x)
=
\int_{\mathcal D}
\kappa_\theta\!\bigl(x,x'\bigr)\,v(x')\,dx'.
\label{eq:no_kernel}
\end{equation}

\textbf{LDS structure in PDE kernels.}
We focus on the case where the kernel has an LDS\textit{-like} structure. A representative example is a geometrically decaying
kernel
\begin{equation}
\kappa(x,x') = c\,\rho^{|x-x'|},\qquad |\rho|<1.
\label{eq:geom_kernel}
\end{equation}
which exhibits stable LDS behavior along the spatial index on a uniform grid:
the interaction depends only on the offset $|x-x'|$ and decays exponentially
with distance. 

This is where truncated USB spectral filtering yields rapidly decaying expansion coefficients and accurate low-mode approximations (App.~\ref{app:theoretical_analysis}); moreover, it remains effective even near marginal stability ($|\rho|\approx 1$) \cite{hazan2017spectral}.

A representative setting where such geometric decay arises is when the discretized kernel is given by the inverse of a \emph{three-point stencil}
Toeplitz operator. Let
\[
(Kv)_i = b\,v_{i-1}+a\,v_i+b\,v_{i+1}, \qquad a>2|b|.
\]
Then \(K\) is invertible, and the entries of \(K^{-1}\) decay geometrically with the grid-index distance \(|i-j|\). In App.~\ref{app:theoretical_analysis}, we prove this result and show that this class naturally arises from standard local,
shift-invariant discretizations of PDE resolvents, such as diffusion-reaction.

\textbf{SFO Kernel.}
These observations motivate parameterizing the kernel directly in the USB basis.
Let \(\{\phi_\ell\}_{\ell=1}^{\infty}\) denote the Hilbert modes such that
\begin{equation}
\label{eq:sfo_spectral}
\kappa_\theta(x,x')
=
\sum_{\ell=1}^{\infty}\theta_\ell\,\phi_\ell(x-x'),
\qquad (x,x')\in\mathcal D\times\mathcal D .
\end{equation}
Prior USB results imply that the coefficients \(\{\theta_\ell\}\) decay rapidly, so the kernel energy concentrates in low-index modes. This motivates the truncated parameterization
\begin{equation}
\label{eq:sfo_spectral_L}
\kappa_{\theta,L}(x,x')
=
\sum_{\ell=1}^{L}\theta_\ell\,\phi_\ell(x-x').
\end{equation}
Theorem~\ref{thm:main_learnability} then implies that this truncation achieves
\(\epsilon\)-accurate approximation with \(L=\tilde{\mathcal O}(\log(1/\epsilon))\) modes.
Thus, for this important subclass of PDE discretizations, SFO achieves \(\epsilon\)-accurate approximation using only logarithmically many USB modes.

In practice, the modes $\{\phi_l\}$ are defined on a discrete grid. 
Since the integral operator in \eqref{eq:no_kernel} with a shift-invariant kernel corresponds to a convolution, it is implemented as a convolution using FFT.

\subsection{SFO Architecture}
Figure~\ref{fig:sfo} provides an overview of the architecture and a single SFO layer.
\begin{figure}[t]
    \centering
\includegraphics[width=.65\textwidth]{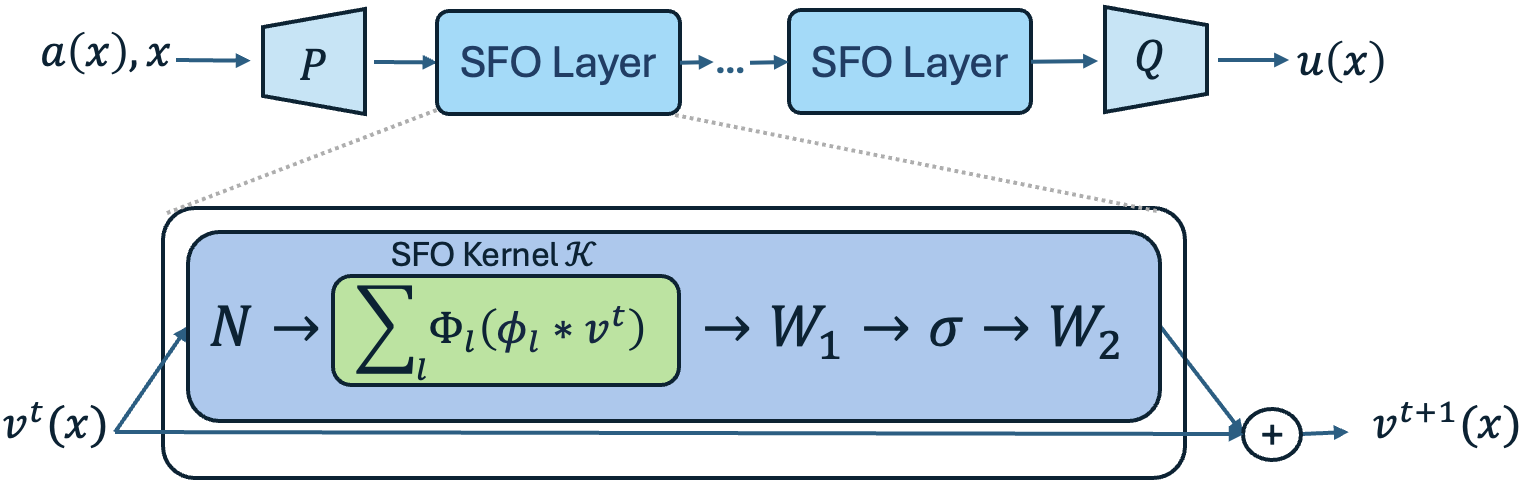}
\caption{SFO architecture. The input field \((a(x), x)\), is lifted by a map \(P\) to a latent representation \(v^0\). The latent field is updated by \(T\) stacked SFO layers, each applying an operator \(\mathcal{K}\). 
A projection \(Q\) maps the last latent state \(v^T\) back to the solution field \(u(x)\).
Within each SFO layer, the current state \(v^t\) is normalized (\(N\)), passed through the operator \(\mathcal{K}\), then through a linear transformation with a nonlinearity, and added back to \(v^t\). Here, \(*\) denotes convolution.
}
    \label{fig:sfo}
\end{figure}

We now instantiate the above approximation inside a kernel-based neural operator. Inserting \eqref{eq:sfo_spectral_L} into the generic integral layer \eqref{eq:no_kernel} gives
\begin{equation}
(\mathcal K_\theta v)(x)
=
\sum_{l=1}^L
\theta_l
\int_{\mathcal D}\phi_l(x-x')\,v(x')\,dx'.
\label{eq:update_continuous_sfo}
\end{equation}

While the LDS analysis motivates the USB parameterization, SFO does not require this assumption. Empirically, the low-mode bias remains effective beyond the LDS setting.


\textbf{Extension to Vector-Valued Fields.}
For vector-valued fields $v:\mathcal D\to\mathbb R^{d}$, the operator uses matrix-valued kernels
\[
\kappa:\mathcal D\times\mathcal D\to\mathbb R^{d\times d}.
\]
In this setting, the $L$ Hilbert modes $\{\phi_l\}_{l=1}^L$ are shared across variables, and the scalar weights $\theta_l \in \mathbb{R}$ are replaced by learned varaible-mixing matrices $\Theta_l \in \mathbb{R}^{d \times d}$,
\begin{equation}
\kappa(x,x')
  = \sum_{l=1}^{L} \Theta_{l}\,\phi_{l}(x-x').
\end{equation}
The corresponding operator is
\begin{equation}
(\mathcal K_\theta v)(x)
=
\sum_{l=1}^L
\Theta_l
\int_{\mathcal D}\phi_l(x-x')\,v(x')\,dx'.
\label{eq:update_continuous_sfo_vector}
\end{equation}

\textbf{Extension to high-dimensional grids.}
For a PDE defined on an $M$-dimensional grid with axes $x_1,\dots,x_M$, the  high-dimensional modes $\{\phi_l(x_1,\dots,x_M)\}_{l=1}^L$ are obtained via a separable \emph{tied-index} tensor product of the 1D modes:
\begin{equation}
\label{eq:tied_index}
\phi_l(x_1,\dots,x_M) = \prod_{m=1}^{M} \phi_l(x_m).
\end{equation}
This yields $L$ globally supported $M$-dimensional modes, where the same index $l$ is shared across all axes.
We ablate the more expressive $L^M$ \emph{multi-index} construction in Section~\ref{sec:ablation}.

\textbf{Convergence as $L \to \infty$.} In \textsc{SFO} we restrict to kernels $\kappa(x,x') = g(x-x')$ with $g \in L^2(\mathcal{D})$. Let $\{\phi_ l\}_{ l\ge1}$ be an orthonormal basis, and define $g_L:= \sum_{ l=1}^{L}\theta_ l \phi_ l$. Then $\kappa_L(x,x') := g_L(x-x')$ is the orthogonal projection of $g$ onto $\mathrm{span}\{\phi_1,\ldots,\phi_L\}$, hence $\|g-g_L\|_{L^2}\to 0$ as $L\to\infty$. In the $M$-dimensional tied-index extension, we obtain $L$ separable modes, and as $L \to \infty$ the approximation converges in $L^2$ to the best tied-index separable approximation.

\textbf{Discrete Form.}
Discretizing the operator in \eqref{eq:update_continuous_sfo_vector}
on a uniform grid $\{x_j\}_{j=1}^{n}$ with spacing $\Delta x$ gives
\begin{equation}
\label{eq:discrete_kernel}
(\mathcal K_\theta v^{t})(x_i)
=
\sum_{l=1}^{L} 
\Theta_l
\sum_{j=1}^{n}
\phi_l(x_i - x_j)\, v^{t}(x_j)\, \Delta x,\qquad
i = 1,\dots,n.
\end{equation}
The inner sum is a discrete convolution between the filter $\phi_l$ and $v^{t}$, so the update can be written as
\begin{equation}
\label{eq:conv_update_discrete_sfo_vector}
(\mathcal K_\theta v^{t})(x_i)
=
\sum_{l=1}^{L}
\Theta_l\,
(\phi_l * v^{\,t})(x_i),
\qquad i = 1,\dots,n,
\end{equation}
which can be computed efficiently via the Fast Fourier Transform (FFT).

\textbf{Compact Parameterization.} Instead of factorizing a learned kernel, we represent $\kappa$ directly in the USB basis and learn only the spectral coefficients $\{\Theta_l\}_{l=1}^{L}$, yielding a compact parameterization.

\textbf{Iterative Updates.}
Kernel-based neural operators follow an iterative update scheme in which the latent field evolves across layers \cite{li2020fourier}. The input $a(x)$ is first lifted to a high-dimensional representation
$v^{0}(x) = P((a(x), x))$,
where $P$ is a linear map. After $T$ updates, the final latent state $v^{T}$ is projected back to the solution space by $Q$, yielding
$u(x) = Q(v^{T}(x))$.

In \textsc{SFO}, each layer $t=0,\dots,T-1$ first normalizes the latent field, then applies an SFO operator $\mathcal K_{\theta}^{t}$, followed by the residual MLP nonlinearity used in~\cite{liu2024flash}.
Let $N(\cdot)$ denote Layer Normalization applied to the latent field. The update rule is
\[
\tilde v^{\,t}(x)=\mathcal K^t_{\theta}\!\left(N(v^t)\right)(x), \qquad
v^{t+1}(x)=v^{t}(x)+W_2^t\,\sigma\!\bigl(W_1^t \tilde v^{\,t}(x)\bigr).
\]

\subsection{Expressivity}
SFO attains high expressivity by expanding the kernel in a fixed orthonormal Hilbert (USB) basis and learning only the expansion coefficients. The only approximation is truncation to the top $L$ modes; empirically, the Hilbert spectrum decays rapidly, so modes beyond $l\!\approx\!20$ are often negligible, and a small rank suffices for an accurate and efficient global kernel. While our formal guarantees are proved for PDEs whose kernel function exhibits LDS
structure (Theorem~\ref{thm:main_learnability}; Appendix~\ref{app:theoretical_analysis}), we find the same structure effective across a broad range of PDE operator-learning benchmarks (Fig.~\ref{fig:hilbert_decay}).

This design contrasts with prior approaches. Unlike SVD-NO, which learns a data-adaptive low-rank basis and can therefore be more sensitive to data coverage, SFO fixes the kernel basis and learns only the expansion coefficients. FNO also uses a fixed basis (Fourier), but its low-frequency truncation imposes an explicit smoothness bias. In contrast, SFO truncates only for efficiency: due to the rapid spectral decay of the USB, the omitted modes are typically negligible. Finally, unlike graph-based operators, where locality requires deep stacking and can cause oversmoothing, the global support of modes enables global information propagation in a single layer.


\subsection{Computational Complexity}
\label{subsec:complexity}
\textbf{Time Complexity.}
Let $n$ be the number of grid points, $d$ the latent dimension, and $L$ the number of modes. An SFO layer \eqref{eq:conv_update_discrete_sfo_vector} computes a FFT of the $d$-channel field $v^t$, costing $O(dn\log n)$, multiplies by the precomputed transforms of $\{\phi_l\}_{l=1}^L$ costing $O(Ldn)$, applies $\Theta_l\in\mathbb{R}^{d\times d}$ and sums over modes in the frequency domain, costing $O(Ld^2n)$, then computes a inverse FFT. Thus, complexity is
\[
O\!\left(dn\log n + Ldn + Ld^2n\right).
\]

\textbf{Memory Complexity.}
Each layer stores latent fields and FFT buffers of size $O(dn)$, filters $\{\phi_l\}_{l=1}^L$ of size $O(Ln)$, and matrices $\{\Theta_l\}_{l=1}^L$ of size $O(Ld^2)$. Thus, the complexity is
\[
O(dn + Ln + Ld^2).
\]


See App.~\ref{app:runtime} and \ref{app:Training_Efficiency} for runtimes and model sizes; SFO is among the fastest and most memory-efficient.
\section{Empirical Evaluation}
\subsection{Experimental Methodology}
\label{sec:exp_method}
We evaluate SFO on diverse benchmarks against strong baselines using the \emph{mean relative $L_2$ error (\%)}, the standard metric in operator-learning benchmarks \cite{pdenneval}:
\begin{equation}
\label{eq:mean_l2_rel_err}
L_2(\%) = 100 \times \frac{1}{N}\sum_{i=1}^{N}
\frac{\lVert u_i - \hat{u}_i \rVert_2}{\lVert u_i \rVert_2}.
\end{equation}  

\subsection{Baseline Models}
We compare SFO to eight baselines: SVD-NO, DeepONet, MNO, Transolver, FNO, UNO, MPNN, and PINO, spanning all neural-operator categories described in Sec.~\ref{related_work}. These representatives are drawn from PDENNEval~\cite{pdenneval} survey together with SVD-NO, MNO and Transolver.

\subsection{Datasets}
\label{sec:datasets}
We evaluate on standard PDE benchmarks from~\cite{pdenneval}, spanning reaction--diffusion (Allen--Cahn, Diffusion--Reaction/Sorption, Cahn--Hilliard), fluid dynamics (Shallow Water), and electromagnetics (Maxwell). Each dataset provides input--solution pairs $\{(a_j,u_j)\}_{j=1}^N$, where $a_j$ is the initial condition and $u_j$ the full trajectory.

The local linearized operators underlying Allen--Cahn, Diffusion--Reaction, and Cahn--Hilliard can exhibit LDS-structured integral kernels, making our theory informative even in these nonlinear settings. App.~\ref{app:theoretical_analysis} presents the linear theory, and Sec.~\ref{sec:nonlinear_theory} discusses its extension to nonlinear PDEs.

Table~\ref{tab:datasets} summarizes the datasets. Additional details are provided in App.~\ref{app:datasets}.
\begin{table*}[t]
\centering
\scriptsize
\setlength{\tabcolsep}{3pt}

\caption{Summary of PDE datasets used in evaluation.}
\scalebox{\tablescale}{%
\begin{tabular}{l l l l}
\toprule
\textbf{Dataset} & \textbf{Equation} & \textbf{Input} & \textbf{Output} \\
\midrule
Allen--Cahn & $\partial_t u - \epsilon \partial_{xx}u + 5u^3 - 5u = 0$ 
& $u(x,0)\in\mathbb{R}^{256}$ 
& $u(x,t)\in\mathbb{R}^{256\times101}$ \\

Diffusion--Reaction & $\partial_t u - 0.5\partial_{xx}u - u(1-u)=0$
& $u(x,0)\in\mathbb{R}^{256}$
& $u(x,t)\in\mathbb{R}^{256\times101}$ \\

Diffusion--Sorption & $\partial_t u - \frac{D}{R(u)}\partial_{xx}u=0$
& $u(x,0)\in\mathbb{R}^{256}$
& $u(x,t)\in\mathbb{R}^{256\times101}$ \\

Cahn--Hilliard & $\partial_t u - \partial_{xx}(10^{-2}(u^3-u)-10^{-6}\partial_{xx}u)=0$
& $u(x,0)\in\mathbb{R}^{1024}$
& $u(x,t)\in\mathbb{R}^{1024\times101}$ \\

Shallow Water & $\partial_t h + \nabla \cdot (h u) = 0$
& $(h,b)\in\mathbb{R}^{128\times128}$
& $(h,u,v)\in\mathbb{R}^{128\times128\times101}$ \\

Maxwell & $
\nabla \cdot E = 0,\nabla \cdot H = 0,
\nabla \times E = - \frac{\mu}{c} \frac{\partial H}{\partial t}, \nabla \times H = \frac{\epsilon}{c} \frac{\partial E}{\partial t}
$
& $(E,H)\in\mathbb{R}^{32\times32\times32\times2\times6}$
& $(E,H)\in\mathbb{R}^{32\times32\times32\times8\times6}$ \\
\bottomrule
\end{tabular}
}
\label{tab:datasets}
\end{table*}

\subsection{Experimental Setup}
Following \cite{li2020fourier}, we train all models using a unified protocol: 500 epochs (200 for Shallow Water), a 90\%/10\% train--validation split, and the \textsc{Adam} optimizer with learning rate $10^{-3}$. Model selection is performed based on validation performance, and all time steps are learned jointly.


\textbf{Hyperparameter Tuning.}
App.~\ref{app:hyper} lists the selected hyperparameters.
For SFO, we tune only the rank $L$ (number of USB modes) and the lifting dimension $d$.
We use $\sigma=\mathrm{gelu}$ and $T=4$ operator layers, following \cite{koren2025svd}.
We perform a grid search over $L\in\{16,20\}$ and $d\in\{32,64,128\}$.
We restrict $L$ to $\{16,20\}$ as the USB spectrum decays rapidly and performance saturates around $L\approx 20$.
Fig.~\ref{fig:loss-vs-l} and Table~\ref{tab:d_sweep_testloss} report ablations over $L$ and $d$.

\textbf{Baseline Hyperparameters.}
All baselines are trained under the same optimization protocol.
We use the data-specific configurations from PDENNEval~\cite{pdenneval} and SVD-NO~\cite{koren2025svd}, and fine-tune MNO~\cite{mno} and Transolver~\cite{transolver} accordingly. 
A local sweep around the reported settings showed no improvement.


\section{Results}
\label{sec:results}
\vspace{-1.5em}
\begin{table*}[ht!]
\centering
\scriptsize
\setlength{\tabcolsep}{2pt}
\caption{Mean $L_2$ relative error (\%). Best results are in \textbf{bold}, second-best are \underline{underlined}. The $\pm$ values denote 95\% confidence intervals. In Diffusion-Sorption, the gaps between methods are on the order of $10^{-3}$; so we report a scaled value $10 \times L_2(\%)$  for readability.}
\scalebox{\tablescale}{%
\begin{tabular}{lccccccccc}
\toprule
         
  & \textbf{SFO}
  & \textbf{SVD-NO}
  & \textbf{FNO}  
  & \textbf{MNO} 
  & \textbf{UNO}
  & \textbf{PINO}
  & \textbf{MPNN}
  & \textbf{Transolver} 
  & \textbf{DeepONet}         \\

\cmidrule(lr){2-10}
\textit{Category} 
& \multicolumn{1}{c}{\textit{USB Kernel}}
& \multicolumn{1}{c}{\textit{SVD Kernel}}
& \multicolumn{1}{c}{\textit{Fourier}}
& \multicolumn{1}{c}{\textit{SSM}}
& \multicolumn{1}{c}{\textit{Fourier}}
& \textit{PINN, Fourier}
& \multicolumn{1}{c}{\textit{Graph}}
& \multicolumn{1}{c}{\textit{Transformer}}
& \multicolumn{1}{c}{\textit{DeepONet}}
\\
\midrule

1D Allen-Cahn
  & \textbf{0.05} $\pm$ 0.012
  & \underline{0.07} $\pm$ 0.007
  & 0.08 $\pm$ 0.001
  & \underline{0.07} $\pm$ 0.021 
  & 0.30 $\pm$ 0.013
  & 0.08 $\pm$ 0.001
  & 0.33 $\pm$ 0.009
  & 0.53 $\pm$ 0.002 
  & 16.5 $\pm$ 0.230                          \\

1D Diffusion-Sorption
  & \textbf{1.08} $\pm$ 0.006
  & \underline{1.09} $\pm$ 0.002
  & \underline{1.09} $\pm$ 0.001
  & \underline{1.09} $\pm$ 0.001
  & 1.10 $\pm$ 0.001
  & 1.09 $\pm$ 0.001
  & 2.86 $\pm$ 0.004
  & 1.14 $\pm$ 0.005
  & 1.14 $\pm$ 0.001\\

1D Diffusion-Reaction 
  & \textbf{0.22} $\pm$ 0.021 
  & \underline{0.33} $\pm$ 0.010
  & 0.39 $\pm$ 0.014
  & 0.38 $\pm$ 0.020
  & 0.40 $\pm$ 0.001
  & 0.43 $\pm$ 0.014
  & 0.39 $\pm$ 0.015
  & 0.44 $\pm$ 0.002
  & 0.61 $\pm$ 0.001           \\

1D Cahn-Hilliard 
  & \textbf{0.08} $\pm$ 0.005
  & 0.47 $\pm$ 0.037
  & \underline{0.13} $\pm$ 0.004
  & 0.18 $\pm$ 0.014 
  & 0.56 $\pm$ 0.013
  & 0.16 $\pm$ 0.010
  & 0.59 $\pm$ 0.011
  & 1.10 $\pm$ 0.035
  & 8.86 $\pm$ 0.149
  \\

2D Shallow Water
  & \textbf{0.38} $\pm$ 0.088
  & \underline{0.39} $\pm$ 0.042
  & 0.49 $\pm$ 0.022
  & 8.15 $\pm$ 4.480
  & 0.52 $\pm$ 0.042
  & 0.46 $\pm$ 0.002
  & 0.50 $\pm$ 0.031
  & 0.61 $\pm$ 0.096 
  & 1.11 $\pm$ 0.235              \\



3D Maxwell
  & \textbf{40.8} $\pm$ 0.103
  &  63.5 $\pm$ 0.020
  & \underline{50.3} $\pm$   0.031
  & 115.3 $\pm$ 3.45
  & 50.4 $\pm$ 0.019
  & 56.1  $\pm$ 0.127
  & 74.4 $\pm$ 0.684
  & 83.6 $\pm$ 0.072 
  & 86.1 $\pm$ 0.596 
  \\
\bottomrule
\end{tabular}
}
\label{tab:results}
\end{table*}

The results in Table~\ref{tab:results} compare SFO against eight SOTA methods on six PDE benchmarks using the mean $L_2$ relative error (\%) \eqref{eq:mean_l2_rel_err}.
For Diffusion-Sorption, the absolute differences between methods are small (on the order of $10^{-3}$), so we report a scaled value $10 \times L_2(\%)$ for readability.

SFO achieves SOTA performance, obtaining error reductions relative to the best-performing baseline on each PDE: $28.57\%$ on Allen-Cahn, $0.9\%$ on Diffusion-Sorption, $33.33\%$ on Diffusion-Reaction, $38.46\%$ on Cahn-Hilliard, $2.56\%$ on Shallow Water, and $18.81\%$ on Maxwell. These results indicate that SFO combines strong accuracy with consistent generalization across diverse PDEs. Training times are reported in App.~\ref{app:runtime}, where SFO is among the fastest methods. In contrast, MNO performs competitively in 1D but degrades on higher-dimensional PDEs, likely due to flattening the spatial grid into a sequence, which weakens locality and harms generalization.

\textbf{Largest gains on PDEs with LDS-structure kernels.}
SFO achieves its largest improvements on Allen-Cahn, Diffusion-Reaction, and Cahn-Hilliard. 
While these PDEs are nonlinear, their local behavior can often be approximated by linear operators with LDS-structured kernels, making the linear theory informative in these settings (App.~\ref{sec:nonlinear_theory}).

\textbf{Significance of Results.}
To assess robustness, each experiment was repeated 10 times with different random seeds. For each PDE, we report the mean and 95\% confidence interval in Table~\ref{tab:results}.

\textbf{Learned Hilbert coefficients.}
To better understand the inductive bias of the USB underlying the truncation, we compute the Frobenius norm $\|\Theta_l\|_F$ for each mode $l$ in every SFO layer and average across layers.
Fig.~\ref{fig:hilbert_decay} shows that the coefficients magnitude decays with $l$, indicating SFO relies on a few modes, supporting the low-rank truncation $L$.
Additional results in all PDEs are reported in App.~\ref{app:hilbert_coeffs}.
Although our formal guarantees apply to settings where the kernel admits an LDS structure (App.~\ref{app:theoretical_analysis}), Fig.~\ref{fig:hilbert_decay} suggests that a similar low-mode bias emerges empirically in diverse PDEs.
\begin{figure}[ht]
  \centering
  \begin{tikzpicture}
    \begin{groupplot}[
      group style={
        group size=3 by 1,
        horizontal sep=1.4cm,
      },
      width=4.1cm,
      height=2.9cm,
      yticklabel style={
      /pgf/number format/fixed,
      /pgf/number format/precision=1,
      /pgf/number format/fixed zerofill
      },
      scaled y ticks=false,
      tick label style={font=\scriptsize},
      title style={font=\scriptsize},
      xlabel style={yshift=4pt,font=\scriptsize},
      ylabel style={yshift=-3pt,font=\scriptsize},
      xlabel={Hilbert mode index $ l$},
      ylabel={Coeff. Magnitude},
      legend style={draw=none,font=\scriptsize},
      legend columns=1,
    ]
      \nextgroupplot[
        title={\textbf{1D Diffusion-Sorption}},
      ]
      \addplot[blue, thick] table[row sep=\\] {
        ell y \\
        1 1.2891567 \\
        2 0.96365774 \\
        3 0.75482148 \\
        4 0.69299412 \\
        5 0.79357636 \\
        6 0.7544446 \\
        7 0.68381041 \\
        8 0.6305126 \\
        9 0.60594189 \\
        10 0.57850027 \\
        11 0.51803517 \\
        12 0.48989668 \\
        13 0.47475231 \\
        14 0.44846418 \\
        15 0.4139401 \\
        16 0.39543915 \\
        17 0.36910275 \\
        18 0.37191623 \\
        19 0.36549753 \\
        20 0.40952578 \\
      };
      \nextgroupplot[
        title={\textbf{1D Diffusion-Reaction}},
      ]
      \addplot[blue, thick] table[row sep=\\] {
        ell y \\
        1 2.0344992 \\
        2 2.00368 \\
        3 1.5870484 \\
        4 1.1051441 \\
        5 1.0237104 \\
        6 1.167325 \\
        7 1.2336972 \\
        8 1.1005946 \\
        9 0.86098617 \\
        10 0.69441676 \\
        11 0.67945987 \\
        12 0.67440712 \\
        13 0.60706371 \\
        14 0.57316959 \\
        15 0.65533209 \\
        16 0.58808571 \\
        17 0.54016036 \\
        18 0.44686103 \\
        19 0.52462894 \\
        20 0.63001537 \\
      };
      \nextgroupplot[
        title={\textbf{2D Shallow Water}},
      ]
      \addplot[blue, thick] table[row sep=\\] {
        ell y \\
        1 0.48258135 \\
        2 0.45599329 \\
        3 0.3127279 \\
        4 0.38140482 \\
        5 0.46915141 \\
        6 0.40440956 \\
        7 0.33070013 \\
        8 0.31235549 \\
        9 0.27859277 \\
        10 0.22541586 \\
        11 0.21419625 \\
        12 0.2352352 \\
        13 0.20438558 \\
        14 0.14561149 \\
        15 0.19403258 \\
        16 0.14084645 \\
        17 0.18578768 \\
        18 0.13477245 \\
        19 0.13838045 \\
        20 0.096348315 \\
      };
    \end{groupplot}
  \end{tikzpicture}
\caption{Magnitude of learned Hilbert spectral coefficients vs.\ mode index $l$. We plot $\|\Theta_l\|_F$ averaged across SFO layers. The decay shows that SFO concentrates on a few low-index modes.}
  \label{fig:hilbert_decay}
\end{figure}
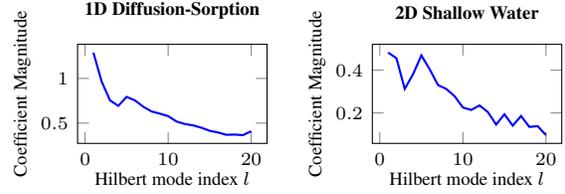

\textbf{Compact parameterization.}
Table~\ref{tab:efficiency_main} shows that across all 
PDEs, SFO uses fewer parameters than SVD-NO, is orders of magnitude smaller than DeepONet and UNO, and is slightly smaller than FNO, PINO, Transolver, and MNO in 1D and 2D and over an order of magnitude smaller in 3D. Combined with competitive runtimes, these results highlight the efficiency of SFO’s compact parameterization.

\begin{table}[ht!]
\centering
\scriptsize
\setlength{\tabcolsep}{3pt}
\caption{Runtime (seconds per epoch) averaged by dimension, with parameter counts in parentheses.}
\label{tab:efficiency_main}
\scalebox{\tablescale}{%
\begin{tabular}{lcccccccc}
\toprule
& \textbf{SFO} 
& \textbf{SVD-NO} 
& \textbf{FNO} 
& \textbf{UNO} 
& \textbf{PINO} 
& \textbf{Transolver} 
& \textbf{MNO} 
& \textbf{DeepONet} \\
\midrule
1D 
& 1.93 (283k) 
& 6.79 (586k) 
& 1.60 (300k) 
& 2.10 (1.09M) 
& 2.29 (300k) 
& 3.05 (326k) 
& 2.96 (311k) 
& 5.77 (3.45M) \\

2D 
& 19.11 (406k) 
& 84.40 (973k) 
& 13.77 (478k) 
& 19.23 (5.98M) 
& 14.55 (478k) 
& 19.13 (605k) 
& 15.47 (721k) 
& 19.81 (429M) \\

3D 
& 71.15 (119k) 
& 180.53 (125k) 
& 58.28 (3.28M) 
& 34.95 (90.22M) 
& 63.21 (3.28M) 
& 75.32 (4.13M) 
& 100.54 (8.98M) 
& 67.56 (52.03M) \\
\bottomrule
\end{tabular}
}
\end{table}

\subsection{Ablation Studies}
\label{sec:ablation}

\begin{table}[ht!]
\scriptsize
\centering
\renewcommand{\arraystretch}{0.9}
\setlength{\tabcolsep}{4pt}
\caption{Mean \(L_2\) relative error (\%) for the orthogonal-basis and GLU ablations. Best results are in \textbf{bold}, second-best are \underline{underlined}. In Diffusion-Sorption, we report \(10 \times L_2(\%)\) for readability.}
\scalebox{\tablescale}{%
\begin{tabular}{lcccc@{\hspace{6pt}}!{\vrule width 1.2pt}@{\hspace{6pt}}cc}
\toprule
& \multicolumn{4}{c}{\textbf{Fixed orthonormal basis}} 
& \multicolumn{2}{c}{\textbf{GLU vs.\ MLP}} \\
\cmidrule(lr){2-5} \cmidrule(lr){6-7}
& \textbf{SFO} 
& \textbf{Fourier} 
& \textbf{Chebyshev} 
& \textbf{Random}
& \textbf{SFO} 
& \textbf{SFO-GLU} \\
\midrule
1D Allen-Cahn
  & \textbf{0.05} $\pm$ 0.01
  & 0.17 $\pm$ 0.01
  & \textbf{0.05} $\pm$ 0.03
  & 0.16 $\pm$ 0.01
  & \textbf{0.05} $\pm$ 0.01
  & \textbf{0.05} $\pm$ 0.01
  \\

1D Diffusion-Sorption
  & \textbf{1.08} $\pm$ 0.01
  & 1.13 $\pm$ 0.02
  & \underline{1.12} $\pm$ 0.01
  & 1.14 $\pm$ 0.00
  & \textbf{1.08} $\pm$ 0.01
  & \textbf{1.08} $\pm$ 0.01
  \\

1D Diffusion-Reaction
  & \textbf{0.22} $\pm$ 0.02
  & \underline{0.24} $\pm$ 0.03
  & \underline{0.24} $\pm$ 0.02
  & \underline{0.24} $\pm$ 0.02
  & \textbf{0.22} $\pm$ 0.02
  & \underline{0.23} $\pm$ 0.03
  \\

1D Cahn-Hilliard
  & \underline{0.08} $\pm$ 0.01
  & 0.12 $\pm$ 0.01
  & \textbf{0.07} $\pm$ 0.01
  & 0.13 $\pm$ 0.00
  & \textbf{0.08} $\pm$ 0.01
  & \underline{0.09} $\pm$ 0.01
  \\

2D Shallow Water
  & \textbf{0.38} $\pm$ 0.09
  & \underline{0.57} $\pm$ 0.06
  & 0.58 $\pm$ 0.04
  & 0.61 $\pm$ 0.05
  & \underline{0.38} $\pm$ 0.08
  & \textbf{0.37} $\pm$ 0.09
  \\


  3D Maxwell
  & \textbf{40.9} $\pm$ 0.10
  & \underline{44.4} $\pm$ 0.10
  & 44.5 $\pm$ 0.16
  & 44.8 $\pm$ 0.44
  & \underline{40.85} $\pm$ 0.103 
  &  \textbf{38.55 }$\pm$ 0.075
\\
\bottomrule
\end{tabular}%
}
\label{tab:ablation_combined}
\end{table}

\textbf{Fixed Orthogonal Basis.}
To isolate the effect of USB, we replace the Hilbert basis in \textsc{SFO} with other fixed global orthonormal bases: Fourier, Chebyshev, and a random orthogonal basis (the $Q$ factor from QR of an i.i.d.\ Gaussian matrix), keeping all settings fixed.
Table~\ref{tab:ablation_combined} shows USB leads on 4/6 benchmarks, ties Chebyshev on Allen-Cahn, and is slightly worse on Cahn-Hilliard; Fourier and random bases consistently underperform.
Since all methods use the same $L$-mode budget, these results suggest that USB’s leading modes capture more useful global structure than the alternatives.


\textbf{GLU vs.\ MLP.} Inspired by \cite{liu2024flash}, we replace the update \( v_t \rightarrow v_{t+1} \) with a GLU gate:
\[
G^{t}=W^{t}_{g}\mathcal{K}^{t}_{\theta}(v^{t}), V^{t}=W^{t}_{v}\mathcal{K}^{t}_{\theta}(v^{t}),\quad
v^{t+1}(x)=v^{t}(x)+\sigma(G^{t}(x))\odot V^{t}(x),
\]


The results in Table~\ref{tab:ablation_combined} show that the performance is similar, 
suggesting that the model’s primary strength comes from the SFO kernel itself, rather than the specific layer connections.

\textbf{Extension to High-Dimensional Grids ($L$ vs.\ $L^M$ modes).}
For PDEs on an $M$-dimensional grid, the 1D Hilbert modes $\{\phi_l\}_{l=1}^L$ can be extended in two ways: \emph{multi-index} and \emph{tied-index}.

In the \emph{multi-index} formulation, we construct a separable mode for each index tuple 
$I=(i_1,\dots,i_M)\in \{1,\dots,L\}^M$:
\[
\Phi_I(x_1,\dots,x_M) = \Pi_{m=1}^M \phi_{i_m}(x_m),
\]
yielding $L^M$ modes. Each mode has an associated mixing matrix $\Theta_I \in \mathbb{R}^{d\times d}$, so both computation and parameter count scale as $O(L^M)$.

In contrast, \emph{tied-index} uses only $L$ modes \eqref{eq:tied_index}, each with a matrix $\Theta_l$, leading to $O(L)$ scaling.

To keep the multi-index experiment tractable, we use $L=6$ ($L^2=36$, $L^3=216$). Table~\ref{tab:high_dim_ablation} shows that multi-index filters slightly worsen performance on 2D Shallow Water and 3D Maxwell despite requiring $3.8\times$ and $14.9\times$ more parameters. This suggests that tied-index modes already capture most useful cross-axis structure, likely aided by the per-layer MLP. We therefore adopt the tied-index formulation, as the additional cost of multi-index bases is not justified empirically.
\begin{table}[ht!]
\scriptsize
\centering
\renewcommand{\arraystretch}{0.9}
\setlength{\tabcolsep}{5pt}
\caption{Mean $L_2$ relative error (\%) for tied-index vs.\ multi-index.}
\scalebox{\tablescale}{%
\begin{tabular}{lcccc|lcccc}
\toprule
& \textbf{$L$} & \textbf{Tied-index} & $L^M$ & \textbf{Multi-index}
& & $L$ & \textbf{Tied-index} & $L^M$ & \textbf{Multi-index} \\
\midrule
2D Shallow Water & \multirow{2}{*}{6} & 0.61 $\pm$ 0.10 & \multirow{2}{*}{36}  & 0.65 $\pm$ 0.08
& 3D Maxwell     & \multirow{2}{*}{6} &  53.78 $\pm$ 1.10 & \multirow{2}{*}{216} & 56.14 $\pm$ 2.83 \\
\# Parameters    &                      & 176,677          &                      & 668,197
& \# Parameters  &                      & 61,960           &                      & 922,120 \\
\bottomrule
\end{tabular}
}
\label{tab:high_dim_ablation}
\end{table}



\textbf{Effect of SFO Rank $L$.}
We sweep the SFO rank $L$ while keeping all other hyperparameters fixed. Figure~\ref{fig:loss-vs-l} shows that test loss decreases sharply as $L$ grows from $4$ to $16$, and then largely plateaus for $L\in[16,32]$. This indicates that the leading USB modes capture most of the operator structure, while additional modes provide only marginal gains. The observed saturation is consistent with the theory
and further suggests that this inductive bias remains useful for PDEs beyond the LDS setting.

\begin{figure}[ht!]
  \centering
  \begin{tikzpicture}
 \begin{axis}[
  width=5cm,
  height=3cm,
  yticklabel style={
  /pgf/number format/fixed,
  /pgf/number format/precision=1,
  /pgf/number format/fixed zerofill
  },
  tick label style={font=\scriptsize},
  title style={font=\scriptsize},
  xlabel={Number of singular functions $L$},
  ylabel={MinMax Test loss},
  ymin=-0.1, ymax=1.1,
  xtick={4,8,12,16,20,24,28,32},
  ylabel style={yshift=-5pt,font=\scriptsize},
  xlabel style={yshift=5pt,font=\scriptsize},
  title={\textbf{Test Loss vs. SFO Rank $L$}},
  legend style={
    at={(1.02,0.5)},      
    anchor=west,
    draw=none,
    font=\scriptsize,
    legend columns=1      
  },
]  
\addplot[green!90!black, thick] table[
      x=L,
      y expr={(\thisrow{loss}-0.000518799)/(0.00106-0.000518799)}
    ] {
L loss
4  0.00106
8  0.00075925
12 0.0007513
16 0.000518799
20 0.00052903
24 0.00065851
32 0.0006570
    };
    \addlegendentry{Allen-Cahn}

    \addplot[green!60!black, thick] table[
      x=L,
      y expr={(\thisrow{loss}-0.0011022140)/(0.0011512-0.0011022140)}
    ] {
L loss
4  0.00114858
8  0.0011512
12 0.0011444
16 0.001107600
20 0.001118932
24 0.001118932
32 0.0011022140
    };
    \addlegendentry{Diffusion-Sorption}
    \addplot[teal, thick] table[
      x=L,
      y expr={(\thisrow{loss}-0.000676445)/(0.00457776-0.000676445)}
    ] {
L loss
4  0.00457776
8  0.0018036329
12 0.0011909
16 0.0009352
20 0.0007163
24 0.00068763
32 0.000676445
    };
    \addlegendentry{Cahn-Hilliard}



\addplot[blue, thick] table[x=L, y expr={(\thisrow{loss}-0.04414862)/(0.0569406071-0.04414862)}] {
L loss
4  0.05452498
8  0.0540688095
16 0.04414862
20 0.044192
24 0.046175056
32 0.0457321
    };
\addlegendentry{Maxwell}
    \end{axis}
\end{tikzpicture}
\caption{Test loss vs.\ rank $L$. Loss drops quickly for small $L$ and then plateaus, showing that the leading modes capture most of the structure. All curves are MinMax normalized to a common scale.}
  \label{fig:loss-vs-l}
\end{figure}

\textbf{Effect of lifting dimension \(d\).}
We sweep the latent width \(d\) with all other hyperparameters fixed. 
Table~\ref{tab:d_sweep_testloss} shows that performance is relatively stable on most datasets.
\begin{table}[ht!]
\scriptsize
\centering
\renewcommand{\arraystretch}{0.9}
\caption{$L_2$ relative error (\%) for varying lifting dimension $d$ across PDEs.}
\scalebox{\tablescale}{%
\begin{tabular}{lccc@{\hspace{10pt}}lccc}
\toprule
& $d=32$ & $d=64$ & $d=128$ 
& & $d=32$ & $d=64$ & $d=128$ \\
\midrule
1D Allen-Cahn         & 0.11 & 0.05 & 0.06 
& 1D Cahn-Hilliard      & 0.21 & 0.08 & 0.11 \\
1D Diffusion-Sorption & 1.12 & 1.08 & 1.10 
& 2D Shallow Water      & 0.40 & 0.38 & 0.39 \\
1D Diffusion-Reaction & 0.24 & 0.22 & 0.27 
& 3D Maxwell            & 40.85 & 55.23 & 59.56 \\
\bottomrule
\end{tabular}
}
\label{tab:d_sweep_testloss}
\end{table}

\section{Conclusion}
We introduced SFO, a neural operator that parameterizes the PDE solution operator via its integral kernel using the Universal Spectral Basis (USB) derived from Hilbert eigenmodes. Our analysis shows that, for an important class of PDEs with LDS-structured kernels, the kernel admits a rapidly decaying USB expansion, enabling accurate approximation with logarithmically many modes. 

Empirically, SFO achieves state-of-the-art performance across six PDEs, reducing error by up to 40\% compared to the strongest prior methods while using fewer parameters. These results demonstrate that representing kernels in a fixed global orthogonal basis provides an effective inductive bias.

More broadly, SFO highlights the potential of combining spectral structure with explicit kernel parameterization for scalable operator learning. 

A limitation of SFO is its reliance on shift-invariant kernels, which may be restrictive for problems with strong boundary effects or non-stationary interactions. Extending SFO to boundary-aware or input-conditioned kernels, as well as to stochastic, inverse, and control settings, is an important direction for future work.

\bibliographystyle{plainnat}
\bibliography{SFO}

\newpage
\appendix

\section{Theory: Learnability of Local Stable Discretizations}
\label{app:theoretical_analysis}

We present a theoretical explanation for why Spectral Filtering Operators (SFO)
can efficiently learn solution operators arising from a broad class of PDE
discretizations.
Our main result shows that inverse operators arising certain stable, local, shift-invariant
discretizations admit approximation by SFO with logarithmic mode complexity.

For simplicity, we consider an infinite grid in which the PDE lives, $\mathbb Z$, as opposed to finite discretization. We discuss the extension to finite grids at the end of this section. For the sake of computational feasibility, we considered a centered finite window on the grid, ignoring the tails and boundary effects.

\subsection{Main theorem}

\begin{theorem}[Learnability of stable local shift-invariant discretizations]
\label{thm:main_learnability}
Let $A$ be the block three-point stencil operator on $\ell_2(\mathbb Z;\mathbb R^d)$ defined by
\[
(Au)[i] := A_0 u[i] + A_1 (u[i-1] + u[i+1]),
\]
where $A_0, A_1 \in \mathbb{R}^{d \times d}$ are symmetric, commuting matrices.
Let $\{(a_j, b_j)\}_{j=1}^d$ denote the pairs of corresponding eigenvalues for $A_0$ and $A_1$.
If the stability condition $a_j > 2|b_j|$ holds for all $j=1,\dots,d$,
then $A$ is invertible and for any $\varepsilon > 0$, there exists an SFO-parameterized operator $M_L$ such that
\[
\|A^{-1} - M_L\| \le \varepsilon, \quad \text{with} \quad L = \widetilde O(\log(1/\varepsilon)).
\]
\end{theorem}

\subsection{Examples}
\label{thm:example}
We show certain examples where the stencil operator in \ref{thm:main_learnability} naturally arises in discretization:
\begin{itemize}
    \item Let $u(x)\in \mathbb{R}^m$ satisfy on $(0, 1)$ the steady coupled diffusion-reaction equation $-Du''(x)+Ru(x) = f(x)$ with Dirichlet boundary conditions $u(0)=u_L,\,u(1)=u_R$. Picking uniform grid $x_0, \ldots, x_{N+1}$ with $x_i - x_{i-1} = h$, we approximate
    $$
    u''(x_i) = \frac{u_{i-1}-2u_i + u_{i+1}}{h^2}
    $$
    Plugging this into the PDE we get for $i \in [N]$ exactly the 3 point stencil form
    $$
    \overbrace{\Big(R + \frac{2}{h^2}D \Big)}^{A_0} u_i + \overbrace{\Big(-\frac{1}{h^2}D \Big)}^{A_1} (u_{i-1} + u_{i+1}) = f_i \; \Rightarrow \; (Au)[i] = f_i
    $$
    Stacking
    $$U =(u_1, \ldots, u_m), \qquad\qquad F=(f_1 -B_0u_L,\; f_2,\; \ldots,\; f_{N-1},\; f_N - A_1u_R)$$ we can write $AU=F$, where $A$ is the tridiagonal block matrix with $A_0$ in the diagonal and $A_1$ in the off diagonals.
    \item Let $c(x,t)\in \mathbb{R}^m$ satisfy on $(0,1)\times (0,\infty)$ the unsteady coupled diffusion--reaction equation
\[
\partial_t c(x,t) = D\,c''(x,t) - R\,c(x,t) + s(x),
\]
with Dirichlet boundary conditions $c(0,t)=u_L,\; c(1,t)=u_R$ and initial condition $c(x,0)=c_0(x)$. Picking the same uniform grid $x_0,\ldots,x_{N+1}$ with spacing $h$, we approximate for $i\in [N]$
\[
c''(x_i,t) \approx \frac{c_{i-1}(t)-2c_i(t)+c_{i+1}(t)}{h^2}.
\]
Plugging into the PDE gives the semi-discrete (method-of-lines) system
\[
\dot c_i(t) = \frac{1}{h^2}D\big(c_{i-1}(t)-2c_i(t)+c_{i+1}(t)\big) - R\,c_i(t) + s_i.
\]
Rearranging yields the same 3-point stencil operator acting on the spatial grid function:
\[
\overbrace{\Big(R+\frac{2}{h^2}D\Big)}^{A_0}c_i(t) + \overbrace{\Big(-\frac{1}{h^2}D\Big)}^{A_1}\big(c_{i-1}(t)+c_{i+1}(t)\big) = s_i - \dot c_i(t)
\;\;\Rightarrow\;\; (Gc(t))[i]= s_i-\dot c_i(t),
\]
where $(Gc)[i]=A_0 c_i + A_1(c_{i-1}+c_{i+1})$. Stacking interior unknowns
\[
C(t)=(c_1(t),\ldots,c_N(t)),\qquad S=(s_1 - A_1u_L,\; s_2,\;\ldots,\; s_{N-1},\; s_N - A_1u_R),
\]
we obtain
$\dot C(t) = -G\,C(t) + S$
where $G$ is as before.

\end{itemize}

\subsection{Proof idea}

The proof decomposes into two conceptual steps.

\paragraph{Step 1: Locality and stability imply exponential decay.}
It is a classical result in numerical analysis and operator theory that the
inverse of a banded, well-conditioned operator exhibits exponential off-diagonal
decay (see, e.g., \cite{demko1984decay,benzi2007decay}).
In our setting, this implies that the inverse of a stable local stencil operator
has a Green's kernel function that decays exponentially with spatial distance.
We include a short proof in the simplest stencil case for completeness.

\paragraph{Step 2: Exponential decay implies spectral filter learnability.}
Results from the spectral filtering literature show that exponentially decaying
impulse responses admit accurate approximation using a small number of
Hankel/USB modes.
In particular, the projection error onto the top spectral filtering modes decays
exponentially, yielding logarithmic mode complexity
\cite{hazan2017spectral}.

\subsection{Setup and Notation}
Let \(M : (\mathbb R^d)^{\mathbb Z} \to (\mathbb R^d)^{\mathbb Z}\) be an operator on vector fields on the infinite discrete grid \(\mathbb Z\), and let \(u:\mathbb Z\to\mathbb R^d\). We say that \(M\) is \emph{translation-equivariant} if
\[
\tau_s M(u) \;=\; M(\tau_s u)\qquad \text{for all } s\in\mathbb Z,
\]
where \(\tau_s\) is the right-shift operator \((\tau_s u)[i] = u[i-s]\). It can be seen that the 3-point stencil operator \(A\) in \ref{thm:main_learnability} is linear and translation-equivariant.

Assume in addition that \(M\) is invertible (as an operator on a suitable sequence space so that the convolutions below are well-defined). Define the matrix-valued impulse response (Green's kernel) of \(M\) by
\[
G \;:=\; M^{-1}\delta_0 \, \text{Id} \;\in\; (\mathbb R^{d\times d})^{\mathbb Z}.
\]
where $\delta_0[i] = \mathbf 1_{(i = 0)}$ is the Dirac distribution.
Equivalently, \(G\) satisfies
\[
M(G) \;=\; \delta_0\, \text{Id}.
\]
Then, for any \(v\) for which the series converges, \(M^{-1}\) acts by discrete convolution with \(G\):
\begin{equation}\label{eq:green_funct}
(M^{-1}v)[i] \;=\; (G * v)[i] \;:=\; \sum_{j\in\mathbb Z} G[i-j]\,v[j].
\end{equation}
We call \(G\) the \emph{Green's kernel} of \(M\).
This terminology mirrors PDE theory: the Green's kernel function \(G(x,s)\) of a linear PDE \(\mathcal L u=f\) satisfies
\begin{equation}\label{eq:green_in_pde}
\mathcal L u=f \quad \Rightarrow\quad
u(x) \;=\; (\mathcal L^{-1}f)(x) \;=\; \int G(x,s)\,f(s)\,ds .
\end{equation}
If \(\mathcal L\) (and its discretization) is translation-invariant, then \(G(x,s)=G(x-s)\), and \eqref{eq:green_in_pde} reduces to a convolution representation, matching \eqref{eq:green_funct}. In this setting, \(M\) discretizes \(\mathcal L\), and \(M^{-1}\) is the discrete solution operator.

\subsection{Lemma 1: Local stability implies exponential decay}

\begin{lemma}[Block stencil yields matrix-valued exponential decay]
\label{lem:block_exp_decay}
Under the assumptions of Theorem G.1, $A$ is linear, translation-equivariant, and invertible, and its Green's kernel function $G$ can be written as a symmetric, matrix-valued exponentially decaying function:
$$ G[t] = \sum_{k=1}^d \Theta_k r_k^{|t|},\qquad |r_k| < 1$$
\end{lemma}
\begin{proof}

Since $A_0,B_0\in\mathbb R^{d\times d}$ are symmetric and commute, they are simultaneously orthogonally diagonalizable: there exists an orthogonal matrix $U$ such that
\[
A_0=U \operatorname{diag}(\alpha_1,\dots,\alpha_d) U^\top, \qquad B_0=U \operatorname{diag}(\beta_1,\dots,\beta_d)  U^\top,
\]
Writing $\tilde u:=U^\top u$ and $\tilde v:=U^\top v$, the equation $Au=v$
decouples into
\[
\alpha_k \tilde u_k[t]+\beta_k\bigl(\tilde u_k[t-1]+\tilde u_k[t+1]\bigr)
=\tilde v_k[t], \qquad k=1,\dots,d.
\]
Equivalently, for each $k$ define the scalar operator on $\ell^2(\mathbb Z)$
\[
(\tilde A^{(k)}w)[t]:=\beta_k w[t-1]+\alpha_k w[t]+\beta_k w[t+1].
\]

Let $S:\ell^2(\mathbb Z)\to\ell^2(\mathbb Z)$ be the shift $(Sw)[t]=w[t-1]$, so $S$ is unitary and $\|S\|=\|S^{-1}\|=1$.
Then
\[
\tilde A^{(k)}=\alpha_k\Bigl(I+\frac{\beta_k}{\alpha_k}(S+S^{-1})\Bigr),
\qquad
\Bigl\|\frac{\beta_k}{\alpha_k}(S+S^{-1})\Bigr\|\le \frac{2|\beta_k|}{\alpha_k}<1.
\]
Hence $I+\frac{\beta_k}{\alpha_k}(S+S^{-1})$ is invertible by the Neumann series, so each $\tilde A^{(k)}$ is invertible on $\ell^2(\mathbb Z)$, and therefore $A$ is invertible on $\ell^2(\mathbb Z;\mathbb R^d)$.

Since $A$ has constant coefficients, it commutes with lattice shifts $(T_su)[t]=u[t-s]$, i.e.\ it is translation equivariant. Since $A$ is linear, translation equivariant, and invertible, it admits a Green's kernel $G : \mathbb Z \rightarrow \mathbb R^{d \times d}$.

In the $U$-basis,
\[
\tilde G[t]:=U^\top G[t]U=\operatorname{diag}(g_1[t],\dots,g_d[t]),
\]
and each $g_k$ satisfies
\[
\beta_k g_k[t-1]+\alpha_k g_k[t]+\beta_k g_k[t+1]=\mathbf 1_{(t=0)}.
\]
If $\beta_k=0$, then $g_k[t]=\mathbf 1_{(t=0)}/\alpha_k$. Assume $\beta_k\neq 0$.
For $t\neq 0$ we have the homogeneous recurrence with characteristic equation
\[
\beta_k r^2+\alpha_k r+\beta_k=0,
\]
with distinct real roots $r_k^\pm$ satisfying $r_k^+r_k^-=1$. Let $r_k$ be the root with $|r_k|<1$, i.e.
\[
r_k=\frac{-\alpha_k+\sqrt{\alpha_k^2-4\beta_k^2}}{2\beta_k}, \qquad |r_k|<1.
\]
$\ell^2$-decay forces $g_k[t]=c_+ r_k^{t}$ for $t>0$ and $g_k[t]=c_- r_k^{-t}$ for $t<0$.
Since both the operator and the forcing $\delta_0$ are invariant under reflection $t\mapsto -t$, if $g_k$ solves the impulse equation then so does $t\mapsto g_k[-t]$; by uniqueness (invertibility of $\tilde A^{(k)}$) we have $g_k[t]=g_k[-t]$, hence $c_+=c_-=:C_k$ and
\[
g_k[t]=C_k\,r_k^{|t|}.
\]
Imposing the equation at $t=0$ gives
\[
\beta_k g_k[-1]+\alpha_k g_k[0]+\beta_k g_k[1]=1
\quad\Longrightarrow\quad
C_k(\alpha_k+2\beta_k r_k)=1.
\]
With the above choice of $r_k$, $\alpha_k+2\beta_k r_k=\sqrt{\alpha_k^2-4\beta_k^2}$, hence
\[
g_k[t]=\frac{r_k^{|t|}}{\sqrt{\alpha_k^2-4\beta_k^2}} \qquad (\beta_k\neq 0).
\]

Finally, we obtain
\[
G[t]
=U\,\operatorname{diag}(g_1[t],\dots,g_d[t])\,U^\top
=\sum_{k=1}^d \Theta_k\, r_k^{|t|},
\qquad
\Theta_k:=\frac{1}{\sqrt{\alpha_k^2-4\beta_k^2}}\,U_{\cdot, k} U_{\cdot, k}^\top,
\]

\end{proof}

\subsection{Lemma 2: Exponential decay implies spectral filter learnability}

\begin{lemma}[Spectral filtering approximation for exponentially decaying kernels
{\citep[Lem.~4.1]{hazan2017spectral}}]
\label{lem:sf_compression}
Let $G:\mathbb Z\to\mathbb R^{d\times d}$ satisfy
$G[t] = \sum_{k=1}^d \Theta_kr_k^{|t|}$ for $|r_k|<1$. Pick $N \geq 1$, and restrict $G$ to $\{0, 1, \ldots, N\}$.
Then for every $\varepsilon>0$, there exists a kernel $g_L$ in the span of the
first $L$ Hankel/USB spectral filtering modes such that
\[
\|g-g_L\|_{\ell_2} \le \varepsilon,
\]
with
\[
L = \widetilde O\!\big(\log(1/\varepsilon)\big).
\]
\end{lemma}

\begin{proof}[Proof sketch]
Exponentially decaying finite sequences are impulse responses of finite-horizon stable linear
dynamical systems.
The spectral filtering framework of \cite{hazan2017spectral} shows that such
responses admit exponentially accurate approximation by projection onto the top
eigenvectors of an associated Hankel matrix, with error controlled by the tail of
its spectrum (Lemma~4.1 therein).
Moreover, the Hankel spectrum decays exponentially
(see Chapter~11 of \cite{hazan2022introduction}), yielding the stated logarithmic mode complexity.

\textit{Remark:} The spectral filtering framework is defined for finite-horizon LDS. However, $G$ is defined on the infinite grid $\mathbb Z$ and is symmetric around 0. Therefore only approximate $G[0:N]$ with truncation error controlled by $N$ due to the exponentially decaying tails of $G$.
\end{proof}

We can now prove Theorem~\ref{thm:main_learnability} as a simple consequence of the proceeding two lemmas: 
\begin{proof}[Proof of Theorem \ref{thm:main_learnability}]
By Lemma~\ref{lem:block_exp_decay}, the inverse operator $A^{-1}$ has an
exponentially decaying Green's kernel function.
Lemma~\ref{lem:sf_compression} therefore yields an SFO-parameterized operator
$M_L$ approximating $A^{-1}$ to accuracy $\varepsilon$ using
$L=\widetilde O(\log(1/\varepsilon))$ modes.
\end{proof}

\paragraph{Remark (Generalization to General Stable Stencils)}

The exponential decay property holds for any stable, local, shift-invariant discretization, not just the three-point stencil. This fact is known in the literature \cite{benzi2007decay}. By itself, exponential decay does not immediately imply learnability by spectral filtering.  However, the exponential decay arising in these cases may be developed into a linear dynamical structure, as we did for the three point stencil, which is then learnable by spectral filtering.

\paragraph{Remark (Marginal Stability and Long-Range Dependencies).}
A key advantage of the spectral filtering framework is its robustness to \textit{marginal stability} ($|\rho| \to 1$).
In this regime, the Green's kernel function decays slowly (e.g., polynomially), corresponding to long-range spatial dependencies common in transport-dominated PDEs (e.g., Shallow Water, Maxwell).
Unlike local basis expansions which struggle to capture such non-local interactions efficiently, the Hankel/USB modes remain highly effective for polynomially decaying kernels, ensuring that SFO performs well even when the system operates near the stability boundary \cite{hazan2017spectral}.

\subsection{Why the linear theory is still informative for nonlinear PDEs.}
\label{sec:nonlinear_theory}
Although Theorem~\ref{thm:main_learnability} is stated for linear shift-invariant discretizations, it still provides useful intuition for nonlinear PDEs. Many nonlinear systems admit informative local linear approximations, and analytical treatments of nonlinear PDEs often rely on linearization to study local behavior or construct tractable approximations. This perspective is also reflected in prior work on nonlinear dynamics, such as \cite{dogariu2025universal}, which suggests that spectral filtering can remain effective beyond strictly linear settings. Therefore, linear approximation can still provide a meaningful and practical foundation for modeling them. A related analogy is gradient descent: its behavior is most transparent in linear or quadratic settings, yet it remains highly effective for nonlinear optimization problems as well.
\section{PDE Benchmarks}
\label{app:datasets}
\textbf{1D Allen-Cahn.}  
\\
\(
  \partial_t u - \epsilon\,\partial_{xx}u + 5u^3 - 5u = 0,\;\epsilon=10^{-4}
  \text{ on }(-1,1)\times(0,1].
\) \\
Input: Initial conditions \(u(x,0)\in\mathbb{R}^{256}\); Output: Full solution \(u(x,t)\in\mathbb{R}^{256\times101}\).

\textbf{1D Diffusion-Reaction.}  
\\
\(
  \partial_t u - 0.5\,\partial_{xx}u - u(1-u) = 0
  \quad\text{on }(0,1)\times(0,1]\;\)(periodic). \\
Input: Initial conditions \(u(x,0)\in\mathbb{R}^{256}\);
Output: Full solution \(u(x,t)\in\mathbb{R}^{256\times101}\).

\textbf{1D Diffusion-Sorption.}  
\\
\(\partial_t u - \frac{D}{R(u)}\,\partial_{xx}u = 0\) on \((0,1)\times(0,1]\);\\
Boundary conditions: \(u(0,t)=1,\;u_x(1,t)=D^{-1}u(1,t)\)
with $D=5\times10^{-4}$ and Freundlich $R(u)$.  \\
Input: Initial conditions \(u(x,0)\in\mathbb{R}^{256}\);
Output: Full solution \(u(x,t)\in\mathbb{R}^{256\times101}\).

\textbf{1D Cahn-Hilliard.}
\\
\(\partial_t u - \partial_{xx}(10^{-2}(u^3 - u) - 10^{-6}\partial_{xx}u) = 0\), \((x, t) \in (-1, 1) \times (0, 1]\) \\
Boundary conditions: \(u(-1,t) = u(1,t),\; \partial_x u(-1,t) = \partial_x u(1,t)\) \\
Input: Initial conditions \(u(x,0) \in \mathbb{R}^{1024}\);
Output: Full solution \(u(x,t) \in \mathbb{R}^{1024 \times 101}\).

\textbf{2D Shallow Water.}  
\\
$\partial_t h + \partial_x(hu) + \partial_y(hv) = 0$ \\
$\partial_t (hu) + \partial_x \left( u^2 h + \frac{1}{2} g_r h^2 \right) + g_r h\, \partial_x b = 0$\\$\partial_t (hv) + \partial_y \left( v^2 h + \frac{1}{2} g_r h^2 \right) + g_r h\, \partial_y b = 0$. \\
Input: \(\bigl(h(x,y,0),\,b(x,y)\bigr) \in \mathbb{R}^{128 \times 128}\);
Output: Full Solution \(\bigl(h,u,v\bigr) \in \mathbb{R}^{128 \times 128 \times 101}\).


\textbf{3D Maxwell’s Equations.} \\
$
\nabla \cdot E = 0, \quad \nabla \cdot H = 0,\quad \nabla \times E = - \frac{\mu}{c} \frac{\partial H}{\partial t}, \quad \nabla \times H = \frac{\epsilon}{c} \frac{\partial E}{\partial t},
$
with $\epsilon = 10$, $\mu = 1$. 
\\
Input: $\bigl(E,H\bigr)\in\mathbb{R}^{32\times32\times32\times2\times6}$; 
Output: $\bigl(E,H\bigr)\in\mathbb{R}^{32\times32\times32\times8\times6}$.

Together, these PDEs span reaction--diffusion (Diffusion-Reaction/Sorption, Allen-Cahn, Cahn-Hilliard), fluid dynamics (Shallow Water), and electromagnetics (Maxwell), providing a broad test bed. Moreover, the local linearized operators underlying Allen-Cahn, Diffusion-Reaction, and Cahn-Hilliard can exhibit LDS-structured integral kernels, making the theory informative in these nonlinear settings; App.~\ref{app:theoretical_analysis} provides the linear theory, and ~\ref{sec:nonlinear_theory} discusses its relevance to nonlinear PDEs.

These datasets are taken from \cite{pdenneval} and processed exactly as in the survey. Each dataset provides input--solution pairs $\{(a_j,u_j)\}_{j=1}^N$, where $a_j$ is the initial condition and $u_j$ the full trajectory.

\section{Analysis of learned Hilbert coefficients.}
\label{app:hilbert_coeffs}
We analyze the learned Hilbert spectral coefficients across PDE benchmarks. For each SFO layer, we compute the Frobenius norm $\|\Theta_ l\|_F$ of the coefficient matrix associated with mode $ l$ and average across layers. Figure~\ref{fig:hilbert_decay_appendix} shows that across 5/6 cases, the coefficients place more weight on low-index modes and decrease with increasing $ l$, providing additional evidence that SFO relies primarily on a small set of global Hilbert modes.

\begin{figure}[ht]
  \centering
  \begin{tikzpicture}
    \begin{groupplot}[
      group style={
        group size=3 by 2,
        horizontal sep=1.5cm,
        vertical sep=1.5cm,
      },
      width=5.0cm,
      height=3.5cm,
      tick label style={font=\scriptsize},
      title style={font=\scriptsize},
      xlabel style={yshift=4pt,font=\scriptsize},
      ylabel style={yshift=-3pt,font=\scriptsize},
      xlabel={Hilbert mode index $ l$},
      ylabel={Coefficient magnitude},
    ]
      \nextgroupplot[
        title={1D Allen-Cahn},
      ]
      \addplot[blue, thick] table[row sep=\\] {
        ell y \\
        1 3.9318442 \\
        2 3.9401011 \\
        3 2.8274977 \\
        4 2.4563622 \\
        5 2.4380355 \\
        6 2.2674465 \\
        7 2.1615543 \\
        8 1.9698896 \\
        9 1.8920445 \\
        10 1.6758238 \\
        11 1.8185999 \\
        12 2.0246053 \\
        13 1.8918247 \\
        14 1.7851831 \\
        15 1.9551882 \\
        16 2.2644284 \\
      };
      \nextgroupplot[
        title={1D Diffusion-Sorption},
      ]
      \addplot[blue, thick] table[row sep=\\] {
        ell y \\
        1 1.2891567 \\
        2 0.96365774 \\
        3 0.75482148 \\
        4 0.69299412 \\
        5 0.79357636 \\
        6 0.7544446 \\
        7 0.68381041 \\
        8 0.6305126 \\
        9 0.60594189 \\
        10 0.57850027 \\
        11 0.51803517 \\
        12 0.48989668 \\
        13 0.47475231 \\
        14 0.44846418 \\
        15 0.4139401 \\
        16 0.39543915 \\
        17 0.36910275 \\
        18 0.37191623 \\
        19 0.36549753 \\
        20 0.40952578 \\
      };
      \nextgroupplot[
        title={1D Diffusion-Reaction},
      ]
      \addplot[blue, thick] table[row sep=\\] {
        ell y \\
        1 2.0344992 \\
        2 2.00368 \\
        3 1.5870484 \\
        4 1.1051441 \\
        5 1.0237104 \\
        6 1.167325 \\
        7 1.2336972 \\
        8 1.1005946 \\
        9 0.86098617 \\
        10 0.69441676 \\
        11 0.67945987 \\
        12 0.67440712 \\
        13 0.60706371 \\
        14 0.57316959 \\
        15 0.65533209 \\
        16 0.58808571 \\
        17 0.54016036 \\
        18 0.44686103 \\
        19 0.52462894 \\
        20 0.63001537 \\
      };
      \nextgroupplot[
        title={1D Cahn-Hilliard},
      ]
      \addplot[blue, thick] table[row sep=\\] {
        ell y \\
        1 2.5906072 \\
        2 2.3156216 \\
        3 2.2359695 \\
        4 2.13448 \\
        5 2.1144969 \\
        6 2.4880245 \\
        7 2.1127443 \\
        8 2.1133435 \\
        9 2.4909306 \\
        10 2.1601806 \\
        11 1.7438855 \\
        12 2.0320044 \\
        13 2.3193636 \\
        14 2.0074909 \\
        15 1.7627034 \\
        16 2.1904857 \\
        17 2.2796738 \\
        18 2.0094874 \\
        19 1.6889458 \\
        20 2.9823902 \\
      };
      \nextgroupplot[
        title={2D Shallow Water},
      ]
      \addplot[blue, thick] table[row sep=\\] {
        ell y \\
        1 0.48258135 \\
        2 0.45599329 \\
        3 0.3127279 \\
        4 0.38140482 \\
        5 0.46915141 \\
        6 0.40440956 \\
        7 0.33070013 \\
        8 0.31235549 \\
        9 0.27859277 \\
        10 0.22541586 \\
        11 0.21419625 \\
        12 0.2352352 \\
        13 0.20438558 \\
        14 0.14561149 \\
        15 0.19403258 \\
        16 0.14084645 \\
        17 0.18578768 \\
        18 0.13477245 \\
        19 0.13838045 \\
        20 0.096348315 \\
      };
      \nextgroupplot[
        title={3D Maxwell},
      ]
      \addplot[blue, thick] table[row sep=\\] {
        ell y \\
        1 0.82454365 \\
        2 1.0228369 \\
        3 1.0880984 \\
        4 1.1698523 \\
        5 1.1825653 \\
        6 1.1246837 \\
        7 1.1603529 \\
        8 1.0221179 \\
        9 0.95661509 \\
        10 0.94955671 \\
        11 0.8759166 \\
        12 0.82723695 \\
        13 0.85766751 \\
        14 0.71817315 \\
        15 0.69946951 \\
        16 0.6292913 \\
        17 0.031619769 \\
        18 0.2652027 \\
        19 0.44308624 \\
        20 0.13204467 \\
      };
    \end{groupplot}
  \end{tikzpicture}
\caption{
Average magnitude of learned Hilbert spectral coefficients across benchmarks.
For each SFO layer, we compute the Frobenius norm $\|\Theta_ l\|_F$ of the coefficient matrix associated with mode $l$ and average across layers. The learned coefficients emphasize low-index modes, supporting the use of a truncated Hilbert expansion.
}
\label{fig:hilbert_decay_appendix}
\end{figure}
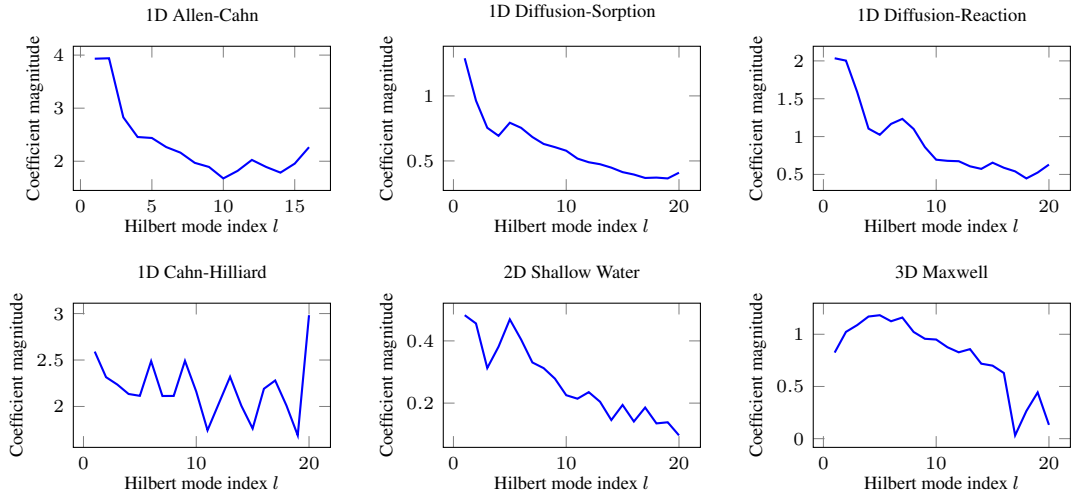

\section{Hyperparameter Tuning.}
\label{app:hyper}

Table~\ref{tab:hyperparameters} summarizes the selected hyperparameters.
For SFO, we tune only the method-specific parameters: the rank $L$ (number of Hilbert modes) and the lifting dimension $d$ (mapping $a$ to $v$).
We fix the activation to $\sigma=\mathrm{gelu}$ and use $T=4$ operator layers, following \cite{koren2025svd}.

We select $L$ and $d$ via grid search over $L\in\{16,20\}$ and $d\in\{32,64,128\}$.
We restrict $L$ to $\{16,20\}$ since the Hilbert spectrum decays rapidly and performance saturates around $L\approx 20$ (Fig.~\ref{fig:hilbert_decay}).
Fig.~\ref{fig:loss-vs-l} and Table~\ref{tab:d_sweep_testloss} report ablations over $L$ and $d$.

\begin{table}[ht!]
\scriptsize
\centering
\caption{Selected \(L\) and \(d\) per PDE}
\label{tab:hyperparameters}
\begin{tabular}{lccc}
\toprule
  & \makecell{\textbf{Rank \(L\)}} 
  & \makecell{\textbf{Lifting Dimension \(d\)}} \\ 
\midrule
1D Allen-Cahn          & 16 & 64 \\
1D Diffusion-Sorption  & 20 & 32\\
1D Diffusion-Reaction  & 20 & 64\\
1D Cahn-Hilliard       & 20 &  64 \\
2D Shallow Water       & 20 & 64 \\
3D Maxwell      & 20 & 32 \\
\bottomrule
\end{tabular}
\end{table}

\section{Compact Parameter Counts.}
\label{app:Training_Efficiency}
We report empirical model size for all baselines. Table~\ref{tab:param_avg_by_dim} summarizes the number of learnable parameters, averaged by PDE dimensionality (1D/2D/3D), providing a compact comparison across problem classes. Overall, \textsc{SFO} is parameter-efficient: it uses fewer parameters than SVD-NO and is orders of magnitude smaller than DeepONet, UNO, and U-Net. Compared to FNO and PINO, \textsc{SFO} is slightly smaller in 1D and 2D and over an order of magnitude smaller in 3D.

\begin{table}[ht!]
\scriptsize
\centering
\renewcommand{\arraystretch}{1.15}
\setlength{\tabcolsep}{6pt}
\caption{Average number of learnable parameters by PDE dimension.}
\begin{tabular}{lccccccc}
\toprule
  & \textbf{SFO} & \textbf{SVD-NO} & \textbf{DeepONet} & \textbf{\makecell{FNO and\\PINO}} & \textbf{UNO}  & \textbf{Transolver}  & \textbf{MNO}\\
\midrule
1D & 283,413 & 586,193 & 3,451,237   & 300,325   & 1,097,051 & 326,629 & 311,781 \\
2D & 406,053 & 973,289 & 429,143,808 & 478,277   & 5,984,051 & 605,717  & 721,877\\
3D & 119,304 & 125,452  & 52,034,992  & 3,287,780 & 90,221,664 & 4,129,985  & 8,984,586 \\
\bottomrule
\end{tabular}
\label{tab:param_avg_by_dim}
\end{table}

\section{Run time}
\label{app:runtime}
Table~\ref{tab:runtime_per_epoch} reports the average wall-clock training time per epoch (seconds per epoch) across PDE benchmarks under an identical training setup.
All methods were trained with the same data splits, optimizer and learning-rate schedule, batch size, and number of gradient updates per epoch.
Overall, FNO, UNO, and SFO are among the fastest methods on most benchmarks.

\begin{table*}[ht!]
\scriptsize
\centering
\renewcommand{\arraystretch}{1.15}
\setlength{\tabcolsep}{6pt}
\caption{Training time per epoch (in seconds) across PDE benchmarks.}
\begin{tabular}{lcccccccc}
\toprule
  & \textbf{SFO} & \textbf{SVD-NO} & \textbf{DeepONet} & \textbf{FNO} & \textbf{UNO} & \textbf{PINO} & \textbf{Transolver}  & \textbf{MNO}  \\
\midrule
1D Allen-Cahn           & 2.09   & 5.87   & 1.09   & 1.61   & 2.06   & 2.57 & 2.96 & 3.36 \\
1D Diffusion-Sorption   & 1.64   & 4.95   & 1.52   & 1.48   & 1.98   & 2.11 & 2.87 & 2.52 \\
1D Diffusion-Reaction   & 2.12   & 13.64  & 1.74   & 1.65   & 2.16   & 2.22 & 3.09& 3.24\\
1D Cahn-Hilliard        & 1.85   & 2.71   & 18.73 & 1.65   & 2.18   & 2.27 & 3.28& 2.72\\
2D Shallow Water         & 19.11  & 84.40  & 19.81  & 13.77  & 19.23  & 14.55 & 19.13 & 15.47\\
3D Maxwell               & 71.15 & 180.53 & 67.56  & 58.28  & 34.95  & 63.21 & 75.32 & 100.54 \\
\bottomrule
\end{tabular}
\label{tab:runtime_per_epoch}
\end{table*}


\section{Robustness to Grid Subsampling.}
We evaluate robustness to grid subsampling by testing \textsc{SFO} on uniformly coarsened grids. Table~\ref{tab:sweep_resolutions} reports performance across grid resolutions. For coarser grids (smaller $n$), we enforce $L \le n$ by capping the number of retained modes.

Overall, subsampling generally increases the $L_2$ relative error compared to the original-resolution setting, with the largest degradation on the higher-dimensional PDEs. However, 
Diffusion--Sorption contains sharp boundary-layer effects that are difficult to learn (App.~\ref{app:bound_conv}).
Therefore, as the grid is coarsened (larger $s$), fewer evaluation points fall in the boundary-layer region and high-frequency boundary effects are smoothed, which can reduce the measured $L_2$ error.

\begin{table}[ht!]
\scriptsize
\centering
\caption{$L_2$ relative error (\%) across grid resolutions. In Diffusion-Sorption, we report $10 \times L_2(\%)$  for readability.}
\begin{tabular}{lccccc}
\toprule
\textit{subsampling factor} & $s=2$ & $s=4$ & $s=6$ & $s=8$ & $s=10$ \\
\midrule
1D Allen-Cahn            & 0.057 & 0.054 & 0.051 & 0.061 & 0.074\\
1D Diffusion-Sorption    &1.092 & 1.068  & 0.808 & 0.565 & 0.4406\\
1D Diffusion-Reaction    & 0.236 & 0.292 & 0.288 & 0.282 & 0.360\\
1D Cahn-Hilliard         & 0.084 & 0.068 & 0.096  & 0.075 & 0.105 \\
2D Shallow Water          & 0.601 & 0.601 &   0.932 & 1.552 & 1.555 \\
3D Maxwell                & 0.038 & 0.043 & 0.049 & 0.058 & 0.062 \\
\bottomrule
\end{tabular}
\label{tab:sweep_resolutions}
\end{table}

For zero-shot resolution transfer, we train on a fixed spatial resolution $n$ and evaluate on uniformly subsampled grids of size $n/s$.
\textsc{SFO} keeps all learned parameters fixed and only reconstructs the Hilbert (USB) spectral basis at test time for the target resolution.

\section{Boundary Handling and Circular Convolution}
\label{app:bound_conv}
SFO evaluates the USB-parameterized integral operator via FFT-based convolution, which requires a shift-invariant kernel of the form $\kappa(x,x')=g(x-x')$ (as in FNO~\cite{li2020fourier} and related variants such as UNO~\cite{uno}, F-FNO~\cite{ffno}, and U-FNO~\cite{ufno}).
On a finite discretized grid, FFT corresponds to circular convolution and thus implicitly assumes a periodic extension of the domain.
This design reflects an efficiency--expressivity trade-off: FFT-based evaluation scales favorably with grid resolution, but restricts the kernel class to shift-invariant convolutions.
In contrast, evaluating a general non-stationary kernel $\kappa(x,x')$ explicitly can be more expressive (e.g., location- or input-dependent interactions), but typically incurs substantially higher computational cost (e.g., as in SVD-NO).


\paragraph{Boundary vs.\ interior error diagnostic.}
Because FFT-based evaluation corresponds to \emph{circular} convolution, it can introduce artifacts near the domain boundaries when the underlying PDE solution is not periodic.
We therefore report the standard relative $L_2$ error on (i) the \emph{full} spatial grid, and additionally on (ii) a \emph{boundary band} consisting of the first and last $b$ grid points, and (iii) the \emph{interior} region consisting of the remaining $N-2b$ points.
Concretely, letting $\hat{u},u\in\mathbb{R}^{N\times C}$ denote the predicted and ground-truth solution, we define
\[
\mathrm{Rel}\,L_2(\Omega) (\%) \;=\; \frac{\|\hat{u}-u\|_{2,\Omega}}{\|u\|_{2,\Omega}} \times 100,
\]
where $\Omega$ is either the full index set $\{1,\dots,N\}$, the boundary set $\{1,\dots,b\}\cup\{N-b+1,\dots,N\}$, or the interior set $\{b+1,\dots,N-b\}$.
We set the band width to $b = 0.1N$ (i.e., $10\%$ of the grid points on each side).
For multi-dimensional grids, we define the boundary band analogously as the set of points within a width-$b$ strip from the domain boundary along \emph{any} spatial axis.
Concretely, for an $M$-dimensional grid of size $n^M$ (so $N=n^M$), let $b=\lfloor 0.1n \rfloor$.

\paragraph{Results.}
Table~\ref{tab:boundary_interior} summarizes the results.
We observe a clear separation between datasets whose solution operators are well-approximated by a \emph{translation-invariant} kernel with \emph{geometric decay} (as motivated in App.~\ref{app:theoretical_analysis}) and those that are more sensitive to boundary behavior.
Specifically, for Allen-Cahn, Diffusion-Reaction, and Cahn-Hilliard, the boundary and interior errors are nearly identical, consistent with the fact that these benchmarks admit a strongly shift-invariant, rapidly decaying effective kernel, for which the circular extension induced by FFT evaluation has limited impact.
In contrast, Diffusion-Sorption and Maxwell exhibit a noticeable gap between boundary and interior error, indicating increased boundary sensitivity and/or less homogeneous operator behavior near the domain edges.
Notably, SFO still attains lower overall error than SVD-NO on these PDEs even though SVD-NO does not impose a shift-invariance or circular-convolution constraint; this suggests that a substantial fraction of the error on this benchmark is not solely attributable to the periodic extension induced by FFT evaluation, but rather reflects intrinsic difficulty near the boundaries.

\begin{table}[t]
\centering
\scriptsize
\setlength{\tabcolsep}{7pt}
\renewcommand{\arraystretch}{1.15}
\caption{Boundary vs.\ interior relative $L_2$ error (in \%). We report the standard full-domain relative $L_2$ error, and additionally the error restricted to a boundary band of width $b$ (first and last $b$ grid points) and the remaining interior region. We set $b=0.1N$. In Diffusion-Sorption, we report $10 \times L_2(\%)$  for readability.}
\label{tab:boundary_interior}
\begin{tabular}{lccc}
\toprule
\textbf{Dataset} &   \textbf{Full} & \textbf{Boundary} & \textbf{Interior} \\
\midrule
1D Allen-Cahn  & 0.052& 0.056 & 0.050 \\
1D Diffusion-Sorption  &  1.119 & 1.401 & 0.451 \\
1D Diffusion-Reaction   & 0.242 & 0.242 & 0.242 \\
1D Cahn-Hilliard &  0.072 & 0.081 & 0.069 \\
2D Shallow Water &  0.394 & 0.099& 0.471
\\
3D Maxwell & 0.040 & 0.042 & 0.034 \\
\bottomrule
\end{tabular}
\end{table}

\end{document}